\documentclass{article}

\usepackage[english]{babel}

\usepackage[letterpaper,top=2cm,bottom=2cm,left=3cm,right=3cm,marginparwidth=1.75cm]{geometry}

\usepackage{amsmath}
\usepackage{amssymb}
\usepackage{amsthm}
\usepackage{graphicx}
\usepackage[colorlinks=true, allcolors=blue]{hyperref}
\usepackage{subcaption}
\usepackage{siunitx}

\usepackage{pgfplots}

\usepackage{listings}
\usepackage{caption}
\usepackage{xcolor}

\theoremstyle{plain}
\newtheorem{theorem}{Theorem}[section]

\newtheorem{definition}[theorem]{Definition}
\newtheorem{corollary}[theorem]{Corollary}

\newtheorem{assumption}{Assumption}

\theoremstyle{remark}
\newtheorem{remark}[theorem]{Remark}

\lstdefinelanguage{Julia}{
  morekeywords={function, end, if, else, elseif, for, while, return, struct, mutable, let, in, begin, try, catch, finally, import, using},
  sensitive=true,
  morecomment=[l]\#,
  morestring=[b]",
}

\DeclareMathOperator{\ReLU}{ReLU}
\DeclareMathOperator{\minmax}{minmax}

\usepackage{listings}
\usepackage{xcolor}

\lstdefinelanguage{Julia}{
  morekeywords={function, return, for, if, else, end, in, while, true, false, struct, mutable, using, import, export, abstract, type, let, local, global, const},
  sensitive=true,
  morecomment=[l]\#,
  morestring=[b]",
  morestring=[b]'
}

\title{On the algorithmic construction of deep ReLU networks}
\author{Daan Huybrechs}

\begin{document}
\maketitle

\begin{abstract}
It is difficult to describe in mathematical terms what a neural network trained on data represents. On the other hand, there is a growing mathematical understanding of what neural networks are in principle capable of representing. Feedforward neural networks using the ReLU activation function represent continuous and piecewise linear functions and can approximate many others. The study of their expressivity addresses the question: which ones? Contributing to the available answers, we take the perspective of a neural network as an algorithm. In this analogy, a neural network is programmed constructively, rather than trained from data. An interesting example is a sorting algorithm: we explicitly construct a neural network that sorts its inputs exactly, not approximately, and that, in a sense, has optimal computational complexity if the input dimension is large. Such constructed networks may have several billion parameters. We construct and analyze several other examples, both existing and new. We find that, in these examples, neural networks as algorithms are typically recursive and parallel. Compared to conventional algorithms, ReLU networks are restricted by having to be continuous. Moreover, the depth of recursion is limited by the depth of the network, with deep networks having superior properties over shallow ones.
\end{abstract}

\section{Introduction}

Neural networks with non-polynomial activation functions can approximate any continuous function~\cite{cybenko1989universal,hornik1989universal}. However, the study of expressivity of neural networks not only describes which functions can be approximated. It also aims to relate function classes to the complexity that is required of a network, such as network depth and network width, to achieve approximation with a certain tolerance within that class. Compared to the study of shallow networks, the study of deep networks is relatively new. We refer to a number of review papers and books for a comprehensive discussion of these aspects and references to recent results~\cite{devore2021,grohs2022expressivity,motamed2024approximation,petersen2024deep}. Without aiming to be complete, function classes for which analysis of deep networks is available include standard smoothness spaces \cite{mhaskar1993multilayered,telgarsky2015representation,yarotsky2017deep}, wavelets and refinable functions~\cite{daubechies2022neural}, hp-approximation spaces and solutions to partial differential equations with singularities~\cite{opschoor2019highorderfem,marcati2020exponential,opschoor2024gevrey}. The approximation power of neural networks in comparison with other nonlinear approximation methods is demonstrated in~\cite{daubechies2022nonlinear}. Rational neural networks for function approximation, involving a rational activation function rather than ReLU, are investigated in~\cite{telgarsky2017rational,boulle2020rational}.

One motivation for such studies is simply to show that neural networks are expressive, in order to shed light on their success in machine learning. Indeed, if networks are very expressive, one can conclude that trained networks can represent very diverse sets of functions. It is another matter, and a largely separate question, to determine whether and how a network can be trained from data to achieve that broad flexibility~\cite{mhaskar2024manifold}. Still, often the analysis in expressivity studies is also constructive, meaning that the methods of proof not only show that functions can be approximated, but also how. The \emph{how} is a challenging question, which is mostly due to the highly composite nature of neural networks -- especially of deep neural networks.
Composition is evident from the standard expression of a feedforward neural network with input $x$,
\[
f(x) = A_1 \sigma( A_2 \sigma(A_3 \ldots +b_3) + b_2) + b_1,
\]
involving weight matrices $A_k$, bias vectors $b_k$, and an activation function $\sigma$ which is understood to act componentwise on a vector. The composition of functions precludes the use of most existing mathematical techniques in approximation theory.

The goal of the paper is to focus on this aspect: how do neural networks represent functions?
The main perspective we adopt is that of a neural network as an algorithm. This is motivated by two main analogies. First, the composition of functions relates to the concept of recursion in an algorithm. In contrast to function composition, recursion is well understood. Second, since a layer in a neural network has many neurons acting independently on their inputs, the layer can be thought of as processing in parallel. In addition to the concepts of recursion and parallelism, we show that the ReLU activation function enables algorithmic control flow using a form of conditional branching.

\subsection{Conditional branching in ReLU network algorithms}

The ReLU activation function is defined mathematically by
\begin{equation}\label{eq:relu}
\ReLU(x) = \max(0,x).
\end{equation}
It is continuous and piecewise linear, and compositions of affine combinations of ReLU maintain that property. ReLU is popular in machine learning not for this property, but for its success in combination with training and optimization approaches. Other activation functions typically also work with varying degrees of success. In contrast, in our algorithmic perspective the precise form of ReLU is crucial.

An algorithmic formulation of~\eqref{eq:relu} has the form of a conditional statement:
\begin{verbatim}
    if x >= 0 then
        x
    else
        0
    end
\end{verbatim}
The two possible branches lead to a different outcome, namely $x$ or $0$, which should be seen as the result of the statement. We will show, formally and by examples, that more general conditions and then-else branches can be realized with a neural network. The above already generalizes when we consider one neuron in a network. Say the overall network has input $x$, then the input of one neuron is a function $a(x)$ of $x$. Its output is mathematically given by $\ReLU(a(x))$ and algorithmically by:
\begin{verbatim}
    if a(x) >= 0 then
        a(x)
    else
        0
    end
\end{verbatim}
In combination with other neurons, we are able to compile more general if-then-else statements.

Yet, one restriction always remains, namely that the output has to be a continuous function of $x$. If it is not, then it can not be realized with a ReLU network. In other words, in that case the conditional expression can not be compiled. Continuity here means that both branches have to agree if $x$ takes a critical value of the condition. Consider the second example, in which the chosen branch depends on the sign of $a(x)$. If $a(x^*)=0$ for a critical value $x=x^*$ then both branches agree: they yield the same value, namely $0$. Due to continuity, it is inconsequential whether the condition is formulated as $a(x) >= 0$ or as $a(x) > 0$, since at the critical value either branch can be taken.

\subsection{Structure of the paper and results}

The expressivity question of this paper is: which types of algorithms can be compiled to a neural network? In order to emphasize the algorithmic perspective, we start in~\S\ref{s:sorting} with what is arguably the most classical type of algorithm in computer science: a sorting algorithm. We describe algorithms for the approximation of functions using geometric folding operations in~\S\ref{s:folding}. To generalize such constructions, we formally consider the scope of conditional branching in neural networks in~\S\ref{s:branching}. Combining the latter two sections enables the `compilation' of certain algorithms into neural networks. We collect a number of examples in~\S\ref{s:examples}. All examples of this paper are illustrated with algorithms, available from an online repository.\footnote{The code is made available in the form of a Julia package \texttt{AlgorithmicNeuralNets.jl}, publicly accessible at \url{https://github.com/daanhb/AlgorithmicNeuralNets.jl}.}

Some results we wish to emphasize in each of these sections are:
\begin{itemize}
    \item in~\S\ref{s:sorting} we construct a network with 112 billion parameters which exactly sorts 16,384 inputs
    \item in~\S\ref{s:folding} we show that a single if-statement in Algorithm~\ref{alg:square} is responsible for exponentially many pieces of a piecewise linear function
    \item in~\S\ref{s:branching} we describe how a fairly broad class of conditional expressions can be compiled to a ReLU neural network
    \item in~\S\ref{s:examples} we construct exponentially converging neural network approximations to monomials, exponentials, trigonometric functions and the multiplication function $f(x,y)=xy$.
\end{itemize}

\subsection{Related work}

Before we expand on the algorithmic point of view, we mention a few more references which are close in spirit to the topic of the paper. The recursive nature of neural networks is emphasized by Mhaskar~\cite{mhaskar2020deep}. The constructive approach of this paper matches well with a \emph{calculus of networks} that was initiated in~\cite{petersen2018calculus}, in which networks representing functions are concatenated in various ways to represent more complicated functions. We adopt these concepts in the software implementation of our network algorithms. Several of the algorithms in this paper employ geometric folding operations, which are also crucial in~\cite{telgarsky2015representation} for the approximation of $f(x)=x^2$ and in~\cite{mccane2018radiallysymmetric} for the approximation of radially symmetric functions. The topic of sorting networks has recently been surveyed in~\cite{petersen2021differentiable}. We return to this reference in more detail in~\S\ref{s:sorting}. Finally, the question of Turing completeness for recurrent neural networks was raised in~\cite{siegelmann1995computationalpower}. We consider only feedforward networks, which can only store state if it has finite size known ahead of time.

\section{A neural network for exact sorting}\label{s:sorting}

Sorting algorithms are a classical topic in computer science~\cite{knuth1998taocp3}. We will start by showing that, seen as a function, the operation of sorting a vector is continuous and piecewise linear. Thus, it can be realized by a neural network using ReLU. The next question is how to do so efficiently.

Among the available sorting algorithms, we find that \emph{bitonic sort}~\cite{batcher1968sorting}, also called the \emph{bitonic sort network}, is a good match to neural networks. The main reason is that bitonic sort was designed to be implemented as a network in hardware. It is a parallel algorithm that processes several input lines simultaneously. In each step, half of the inputs are compared to the other half and sorted pairwise. The basic building block for this operation is a \emph{comparison element}, which takes two inputs and outputs them in sorted order. These comparison elements are structured like a network as shown in Fig.~\ref{fig:bitonic}. It is a feature of the method that the locations of the comparison elements are deterministic, i.e., they are non-adaptive. That property is why the network can be implemented in hardware. It is also why it can be realized with a neural network.

Among other sorting networks, bitonic sort aims to minimize the number of steps in the algorithm, since each step increases the latency in the hardware implementation. In our setting that corresponds exactly to minimizing the depth of the network. An input of dimension $N$ requires $\frac12 N (\log N)^2$ layers. Bubble sort can also be seen as a sorting network, but it would require ${\mathcal O}(N^2)$ layers.

The observation we make above about the compatibility of sorting networks with neural networks was already made in literature. Bitonic sort was proposed by Petersen et al in~\cite{petersen2021differentiable} precisely for its efficiency and its compatibility with neural networks. We also refer to the references in~\cite{petersen2021differentiable} for an overview of other recent uses of sorting functionality in neural networks. One property shared among all these references, which Petersen et al explicitly observe, is that the neural networks described in literature only sort approximately. The objective in the setting of~\cite{petersen2021differentiable} is not to sort inputs exactly, but to approximately maintain a relative order of inputs while learning absolute sizes. To that end, the authors advocate the use of smoothed $\min$ and $\max$ functions, which are compatible with gradient-based optimization and training. By realizing the $\min$ and $\max$ functions exactly, we show that sorting with neural networks can also be exact.

\subsection{Exact neural representation of min and max}

\begin{figure}[t]
  \centering

  \begin{subfigure}[b]{0.45\textwidth}
    \centering
    \includegraphics[width=\textwidth]{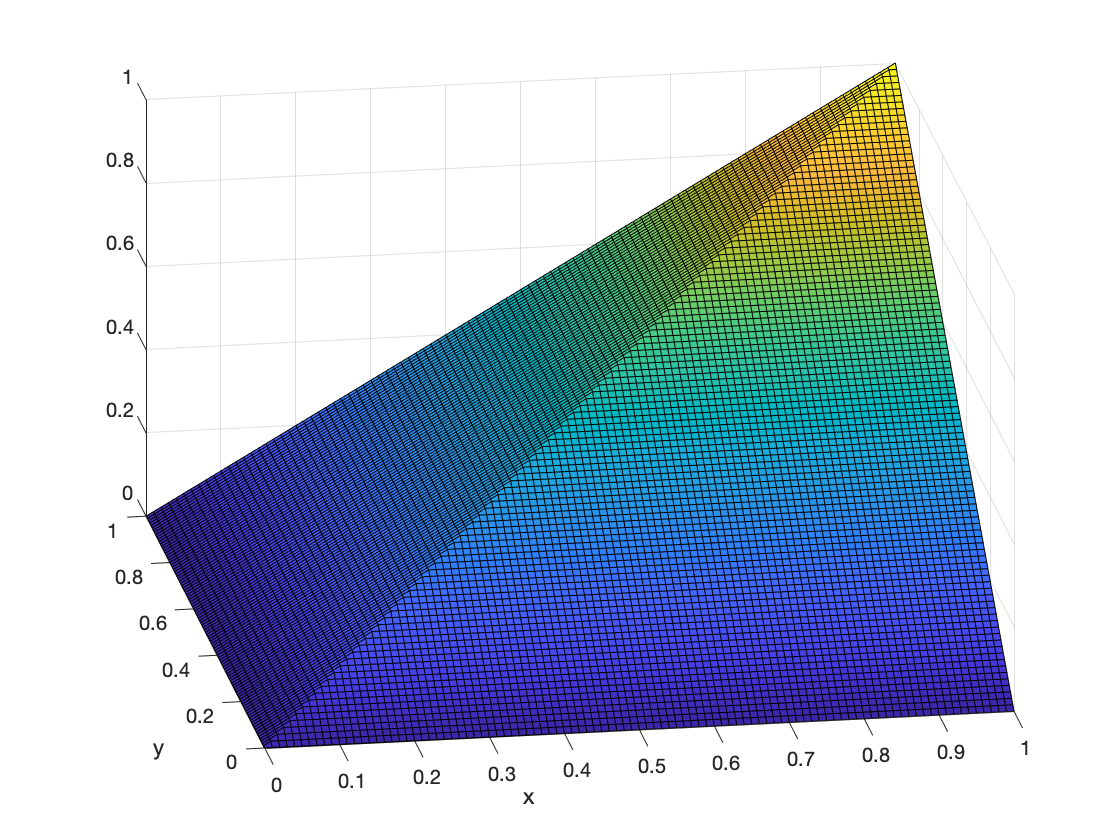}
    \caption{$\min(x,y)$}
    \label{fig:min}
  \end{subfigure}
  \hfill
  \begin{subfigure}[b]{0.45\textwidth}
    \centering
    \includegraphics[width=\textwidth]{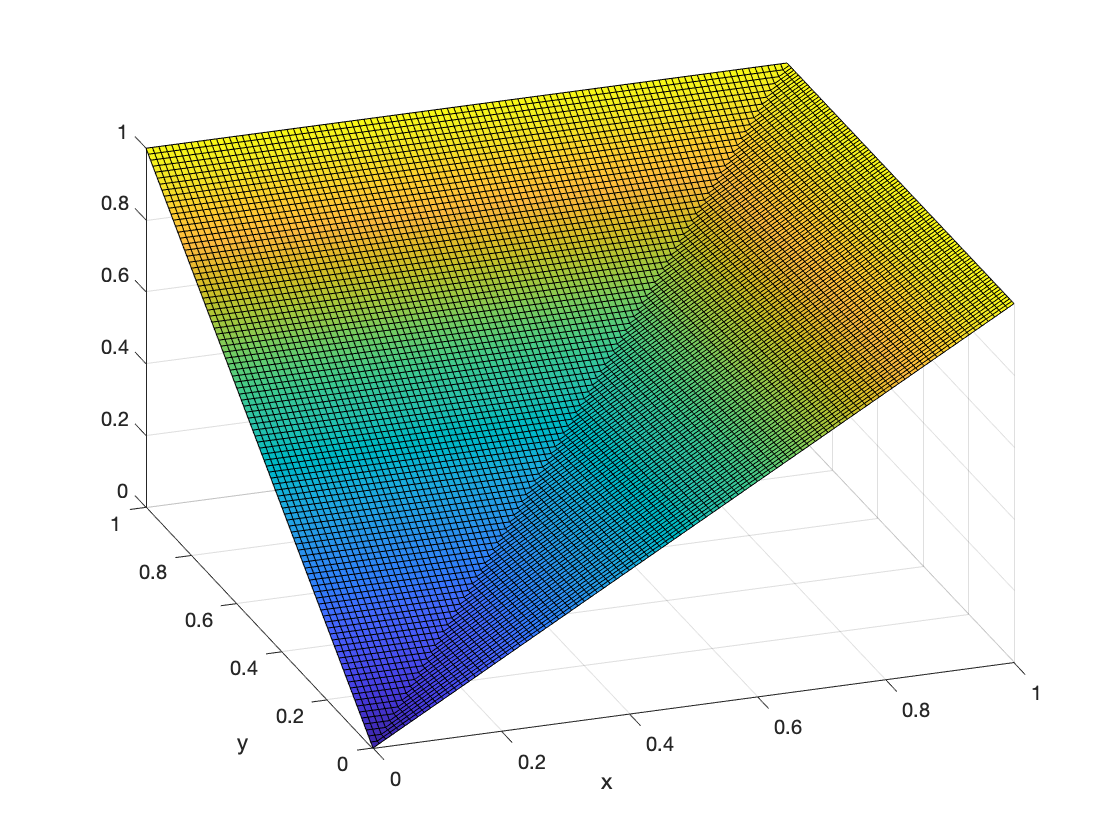}
    \caption{$\max(x,y)$}
    \label{fig:mac}
  \end{subfigure}

  \caption{The $\min$ and $\max$ functions are continuous and piecewise linear functions of $x$ and $y$.}
  \label{fig:minmax}
\end{figure}

The bivariate functions $\min(x,y)$ and $\max(x,y)$, with real arguments $x,y \in \mathbb{R}$, are continuous and piecewise linear in $x$ and $y$.  This property is illustrated in Fig.~\ref{fig:minmax}. In order to find its neural representation, we first formulate $\max$ as a conditional statement. This results in the simple algorithm shown in Algorithm~\ref{alg:max}.\footnote{The listing is meant to convey the algorithm in pseudocode. However, it is also valid Julia syntax.} The algorithm involves a single if-then-else statement.

\begin{lstlisting}[float, caption={An algorithm for the bivariate function $\max(x,y)$}, label={alg:max}]
function max(x, y)
    if x >= y then
        return x
    else
        return y
    end
end
\end{lstlisting}

The next step is to realize the same outcome using a ReLU neuron with a particular input. This can be achieved by subtracting $y$ from both branches of the statement, and adding it back afterwards. Thus, the conditional statement becomes
\begin{verbatim}
    if x-y >= 0 then
        x-y
    else
        0
    end
\end{verbatim}
We add $y$ to the result of this expression. In terms of ReLU this can be expressed as
\[
 \max(x,y) = \ReLU(x-y) + y.
\]
This formulation uses a single neuron with input $x-y$. However, it also requires a skipping connection for $y$. Skipping can be avoided by adding two more neurons:
\[
 \max(x,y) = \ReLU(x-y) + \ReLU(y) - \ReLU(-y).
\]

The $\min$ function is similar, and there are several possibilities. Aiming for a minimal number of neurons in total for $\min$ and $\max$, we like to reuse the two neurons involving $y$. This leads to
\[
 \min(x,y) = -\ReLU(y-x) + y = -\ReLU(y-x) + \ReLU(y) - \ReLU(-y).
\]

We now proceed to an explicit description of the weights. We have used the identity function $y = \ReLU(y) - \ReLU(-y)$ to avoid a skipping connection. This corresponds to a representation involving two weight matrices $A^{\textrm{Id}}_1$ and $A^{\textrm{Id}}_2$,
\[
 x = A^{\textrm{Id}}_2 \ReLU\left( A^{\textrm{Id}}_1 x\right) = \left[\begin{array}{cc}
 1 & -1\end{array}\right] \ReLU\left(\left[\begin{array}{c}
 1 \\ -1\end{array}\right] x \right).
\]
Here we adopt the typical convention that ReLU acts componentwise on a vector.

A min-max comparison element has two inputs, which we denote by $\mathbf{x} = \left[\begin{array}{cc}
 x\\
 y
\end{array}\right]$.
The formulas above lead to the representation
\begin{equation}\label{eq:minmax}
    \minmax(\mathbf{x}) = \left[\begin{array}{cc}
 \min(x,y)\\
 \max(x,y)
\end{array}\right] = A_1 \, \ReLU\left(A_2 \mathbf{x} \right),
\end{equation}
with
\[
A_2 = \left[\begin{array}{cc}
 1 & -1\\
 -1 & 1\\
 \mathbf{0} & A^{\textrm{Id}}_1
\end{array}\right] = \left[\begin{array}{cc}
 1 & -1\\
 -1 & 1\\
 0 & 1 \\
 0 & -1
\end{array}\right] \quad \mbox{and} \quad A_1 = \left[\begin{array}{ccc}
 0 & -1 & A^{\textrm{Id}}_2 \\
 1 & 0 & A^{\textrm{Id}}_2 
\end{array}\right]  = \left[\begin{array}{cccc}
 0 & -1 & 1 & -1 \\
 1 & 0 & 1 & -1
\end{array}\right].
\]
We have used block-matrix notation to highlight the structure of the matrices. The first row of $A_2$ encodes the input $x-y$ to the first neuron, while the second row encodes $y-x$ for the second neuron. The next two rows encode the required input for the identity neurons. The matrix $A_1$ rearranges the neuron outputs according to the formulas for $\min$ and $\max$. The matrices have size $4 \times 2$ and $2 \times 4$ respectively. There are no bias vectors in this example, or equivalently they are zero.

\subsection{Exact neural representation of the bitonic network}

\begin{figure}{t}
\begin{center}
 \includegraphics[width=0.8\linewidth]{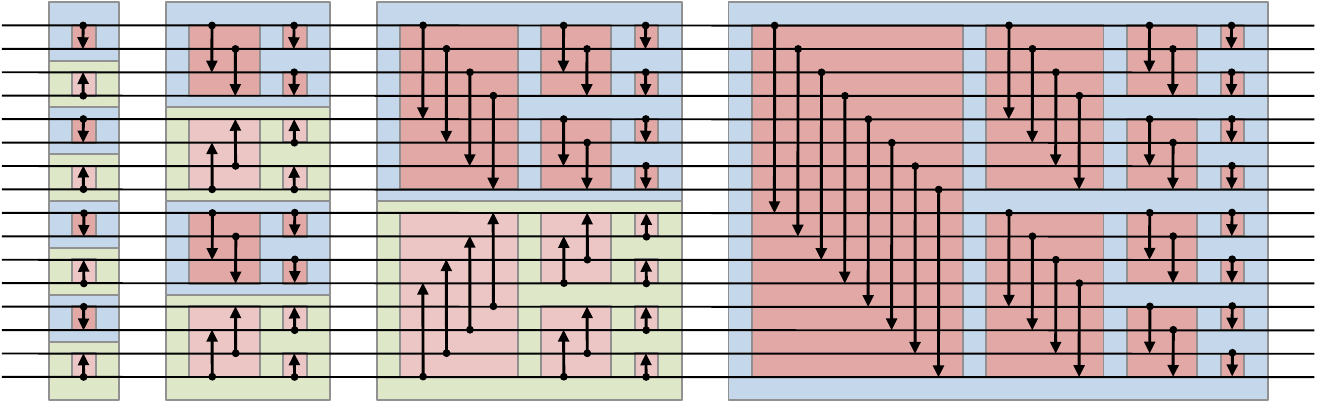}
\end{center}
\caption{Structure of the bitonic network (image from \cite{wikipedia-bitonic}). Each arrow represents a comparison element which operates on two lines and sorts them in ascending or descending order. The arrows represented in parallel in the same coloured box are also executed in parallel. This network has input dimension $16$ and it has $10$ sorting layers.}\label{fig:bitonic}
\end{figure}

We will not explain here in full how bitcoin sort operates. Knuth in \S5.2.2 of \cite{knuth1998taocp3} refers to it as \emph{Batcher's parallel method} and writes: ``His method is not at all obvious; in fact a fairly intricate proof is needed just to show that it is valid, since only comparatively few comparisons are made''. We also refer to the original paper by Batcher~\cite{batcher1968sorting} for a description of the method.

We do illuminate the structure of the algorithm. To that end we show in Fig.~\ref{fig:bitonic} the network structure of bitonic sort. This is also precisely the structure of the neural network we propose using the comparison element~\eqref{eq:minmax}. As shown in the figure, each comparison element connects two lines. To be more precise in our description, Algorithm~\ref{alg:bitonic_sort} shows a complete implementation of bitonic sort and details the exact locations of the comparison elements, which are independent of the input data.\footnote{This listing is valid Julia code. It is meant to be precise and reproducible, and hopefully also informative without knowing Julia. However, it can no longer be considered pseudocode. Note that we do not claim that we can automatically compile this algorithm to a neural network. We claim that a neural network can be made which realizes exactly this algorithm. In later examples, the \emph{compilation} sometimes relies on analytic information which is not made explicit in the listing of the algorithm, and hence compilation can not (yet) be fully automated. Still, the algorithms offers a complete understanding of what the neural network does.} We assume for simplicity that the input vector has dimension $N = 2^L$.

\subsection{Computational complexity and timings}

From the algorithm we can make some complexity statements. There are three loops. The two outer ones with indices $i$ and $j$ have at most $L=\log_2(N)$ iterations. The inner variable $k$ loops over the length of the vector $N$ and decides in a non-intuitive, yet deterministic and non-adaptive, way to exchange two elements after their comparison. The inner loop performs at most $N$ exchanges, and one can see that it is executed $\frac12 L(L+1)$ times. The latter is the number of hidden layers of the network, while the $N$ exchanges can be realized in parallel in one layer. Since a comparison element has $4$ neurons, and since there are exactly $N/2$ comparisons per layer, the width of a layer is $2N$.

\begin{lstlisting}[float=t,caption={Bitonic Sort in Julia.}, label={alg:bitonic_sort}]
minmax(x,y) = (min(x,y), max(x,y))
maxmin(x,y) = (max(x,y), min(x,y))

# In-place sorting of an input vector x
function bitonic_sort!(x)
    N = length(x)
    L = Int(log2(N))  # we assume that N = 2^L
    for i in 1:L
        for j in i-1:-1:0
            for k in 0:N-1
                l = xor(k,2^j) # bitwise XOR of k and 2^j
                if l > k
                    if (k & 2^i) == 0  # bitwise AND of k and 2^i
                        x[k+1],x[l+1] = minmax(x[k+1],x[l+1])
                    else
                        x[k+1],x[l+1] = maxmin(x[k+1],x[l+1])
                    end
                end
            end
        end
    end
end
\end{lstlisting}

Summarizing, we can make the following statements about the size of the network:
\begin{enumerate}
 \item The input and output layers have width $N = 2^L$.
 \item There are $\frac12 L(L+1)$ hidden layers with $2N$ neurons.
 \item The network has depth ${\mathcal O}(\log(N)^2)$ and width $2N$.
\end{enumerate}
The weight matrices between the hidden layers have size $2N \times 2N$.
There are also input and output weight matrices of size $2N \times N$ and $N \times 2N$ respectively, and bias vectors of size $2N$ for the hidden layers and $N$ for the final one. This means that in principle the total number of parameters in the network is:
\[
P = 4N^2 + 3N + \left(\frac12L(L+1)-1\right)(4N^2+2N) = {\mathcal O}(N^2 \log(N)^2).
\]
Yet, the matrices are highly sparse and the biases are all zero. A similar but more involved computation as the one above yields the number of nonzero parameters in the network to be
\[
 P_{\textrm{sparse}} = 6N+9N\left(\frac12L(L+1)/2-1\right) = {\mathcal O}(N \log(N)^2).
\]

\begin{table}[ht]
\centering
\begin{tabular}{c|c|c|c|c|c}
$L$ & $N$ & $P$ & $P_{\textrm{sparse}}$ & Construction time & Evaluation time \\
\hline
4  & 16 &  10k &  1.4k &  \SI{193}{\micro\second}  & \SI{5.6}{\micro\second} \\
6  & 64 &   346k & 11.9k  & \SI{2.6}{\milli\second}   & \SI{38.2}{\micro\second} \\
8  & 256 &  9.4m &  82.2k &  \SI{80}{\milli\second}  & \SI{169}{\micro\second} \\
10  & 1024 & 230m &  504k &  \SI{1.3}{\second} & \SI{1.0}{\milli\second} \\
12  & 4096 & 5.2b &  2.9m &  \SI{24}{\second} & \SI{5.5}{\milli\second}  \\
14 & 16384 & 112b & 15.4m & \SI{467}{\second} & \SI{30.1}{\milli\second} \\
\end{tabular}
\caption{Performance metrics of neural bitonic sort for different input sizes. Timings were ran on a desktop with a 3.4GHz Intel i5-3570 CPU and with 16GB RAM.}
\label{tab:bitonic}
\end{table}

In our implementation we have implemented the weight matrices as sparse matrices. This speeds up both the construction and the application of the network. A few performance numbers are listed in Table~\ref{tab:bitonic}. In addition to dense and sparse network size, we have included the time it takes to construct the network and the time to evaluate it for a given input vector of length $N$. Neither operation was optimized and the code is simply ran on a ten-year-old CPU. The construction time is not an accurate reflection of its theoretical computational complexity since the networks become large and various implementation factors are at play. The timings are included to demonstrate that it is feasible to make huge networks on modest hardware. The largest network has $107$ layers and $112$ billion parameters in total, of which $15.4$ million are nonzero. It is constructed in less than 10 minutes. The evaluation time does scale favourably with $N$.

Obviously the neural network algorithm is much slower than quicksort and other conventional sorting algorithms. Its realization here is meant to illustrate the point that a neural network can implement an algorithm. It is a deep neural network which we can completely understand.

\section{Function approximation by folding and unfolding}\label{s:folding}

We now turn our attention to the approximation of functions $f(x)$ by a network $\hat{f}(x)$ with input $x$. We use a geometric folding operation to construct recursive algorithms, which leads to explicit neural networks which we can understand and analyze.

\subsection{A special case: the approximation of $x^2$ on $[0,1]$}

The function $f(x) = x^2$ is not piecewise linear. Hence, it can only be approximately represented by a ReLU network. It turns out that this function is a special case which allows a very compact ReLU representation. This was described by Telgarsky in~\cite{telgarsky2015representation}. Later on, Yarotsky in~\cite{yarotsky2017deep} described how to approximate the multiplication function $f(x,y) = xy$ starting from the identity
\begin{equation}\label{eq:multiplication}
 xy = \frac12 \left( (x+y)^2 - x^2 - y^2 \right)
\end{equation}
and using Telgarsy's result. In turn, having an approximation to the multiplication operation allows one to approximate higher degree polynomials and many other smooth functions. This construction is the basis of several analytical (and often constructive) techniques to approximate more general classes of functions. It underlies, e.g., the analysis of hp-spaces and spaces of singular functions in~\cite{opschoor2019highorderfem,opschoor2024gevrey}. Thus, while $x^2$ is a special case, it is an important special case.

We will use the notation of Yarotsky to describe Telgarsky's approximation to $x^2$. First we define the hat function
\begin{equation}\label{eq:hat}
g(x) = \left\{\begin{array}{cc}2x,  & 0 \leq x < \frac12 \\ 
2(1-x), & \frac12 \leq x \leq 1 \\
0, & \mbox{otherwise.}\end{array} \right.
\end{equation}
The support of this hat function is the interval $[0,1]$, and it has its maximum value $1$ in the midpoint $x=1/2$. Telgarsky motivates this function as a mirror function which maps both $[0,1/2]$ and $[1/2,1]$ onto $[0,1]$. It is continuous and piecewise linear and can be implemented with three neurons. Next, the \emph{sawtooth function} $g_s(x)$ is the repeated application of $g$ to itself $s$ times:
\begin{equation}\label{eq:sawtooth}
 g_s(x) = \underbrace{g \circ g \circ \ldots \circ g}_{s}(x).
\end{equation}
The function $g_s(x)$ consists of $2^{s-1}$ dilated copies of a hat function on the interval $[0,1]$. Since it is a composite function, it can be implemented by a network with $s$ layers, each applying $g$.

The approximation to $x^2$ is~\cite[Lemma 2.4]{telgarsky2015representation}
\begin{equation}\label{eq:square}
 x^2 \approx f_s(x) = x - \sum_{s=1}^L \frac{1}{2^{2s}} g_s(x).
\end{equation}
This yields a piecewise linear approximation $f_L(x)$ to $x^2$ with $2^L$ pieces. It interpolates $x^2$ in the points $k \, 2^{-L}$ for $k=0,1,\ldots,2^L$.

One can verify that the approximation to $x^2$ is correct, but so far it has not led to similar approximations for other functions. Indeed, it seems a remarkable coincidence that $x^2$ can be written as a sum of sawtooth functions and a linear term. How special is this special case?

\subsection{Folding and unfolding with graded meshes on $[0,1]$}

\begin{figure}
    \begin{center}
    \includegraphics[width=0.8\linewidth]{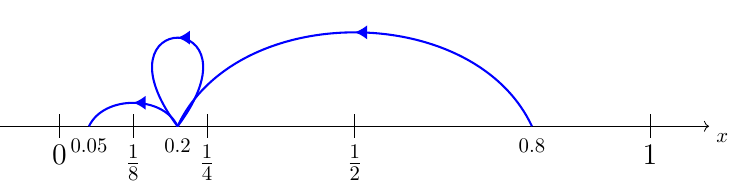}
    \end{center}
\caption{Illustration of the repeated foldings. First, the point $x = 0.8$ is reflected with respect to pivot $\frac12$ onto $h_1(x)=0.2$. Next, there is a loop at $0.2$ because $0.2$ is already smaller than the next pivot $\frac14$. Finally, $0.2$ is reflected with respect to $\frac18$ onto $h_3(0.2)=0.05$.}\label{fig:folds}
\end{figure}

The structure of the approximation is more visible when applied in reverse. We will describe the same construction using a folding operation, rather than mirroring. To that end, we define a \emph{folding function}
\[
 h_1(x) = \left\{\begin{array}{cc}x,  & 0 \leq x < \frac12, \\ 
1-x, &  \frac12 \leq x \leq 1.\end{array} \right.
\]
The effect of $h_1$ is that it folds $[0,1]$ onto $[0,\frac12]$. One can think of the midpoint $x=\frac12$ as the pivot of the folding. Note that $h_1(x)$ is continuous and piecewise linear.

\begin{lstlisting}[float=t,caption={Recursive approximation to $x^2$.}, label={alg:square}]
function square_approx(x, j = 1, L = 20)
    if j == L+1
        return 0.0      # alternative: return x/2^L (interpolating)
    end
    pivot = 1/2^j
    z = square_approx(pivot - abs(x-pivot), j+1, L)
    if x > pivot
        return z + 4*pivot*x - 4*pivot^2
    else
        return z
    end
end
\end{lstlisting}

Consider the function $f(x) = x^2$. We aim to relate $f(x)$ and $f(h_1(x))$. For $x \in \left[\frac12,1\right]$ we have
\begin{equation}\label{eq:square_relate}
 f(h_1(x)) = f(1-x) = (1-x)^2 = x^2 - 2x + 1 = f(x) - 2x + 1.
\end{equation}
For $x \in [0,\frac12]$ nothing happens since $h_1(x)=x$. Therefore,
\begin{equation}\label{eq:unfold1}
 f(x) = \left\{\begin{array}{cc}f(h_1(x)),  & 0 \leq x < \frac12, \\ 
f(h_1(x)) + 2x - 1, &  \frac12 \leq x \leq 1.\end{array} \right.
\end{equation}

The point is the following: we have reduced the approximation problem from $[0,1]$ to $\left[0,\frac12\right]$. If we succeed in approximating $f(x)$ on $[0,\frac12]$ then by \emph{unfolding} following~\eqref{eq:unfold1} we can also find an approximation on the larger interval $[0,1]$.

We can repeat the folding operation using smaller and smaller pivots. We define the folding function $h_j(x)$ with pivot $2^{-j}$ as
\[
 h_j(x) = \left\{\begin{array}{cc}x,  & 0 \leq x < 2^{-j}, \\ 
2^{-j+1}-x, &  2^{-j} \leq x \leq 2^{-j+1}.\end{array} \right.
\]
This is illustrated in Fig.~\ref{fig:folds}. An equivalent and perhaps clearer expression is
\[
h_j(x) = 2^{-j} - | x - 2^{-j}|.
\]
The operation maps $[0,2^{-j+1}]$ onto $[0,2^{-j}]$. Our reasoning is the same as above. For $x \in [2^{-j},2^{-j+1}]$ we have
\[
 f(h_j(x)) = f(2^{-j+1}-x) = (2^{-j+1}-x)^2 = x^2 - 2^{-j+2}x + 2^{-2j+2} = f(x) - 2^{-j+2}x + 2^{-2j+2}.
\]
We conclude that $f(x) = f(h_j(x)) + \eta_j(x)$ with the piecewise correction term given by
\begin{equation}\label{eq:square_correction}
 \eta_j(x) = \left\{\begin{array}{cc}0,  & 0 \leq x < 2^{-j}, \\ 
2^{-j+2}x - 2^{-2j+2}, &  2^{-j} \leq x \leq 2^{-j+1}.\end{array} \right.
\end{equation}
Equivalently,
\[
 \eta_j(x) = 2^{-j+1} \, |x-2^{-j}| + 2^{-j+1}x - 2^{-2j+1}.
\]
Note that $|x| = \ReLU(x) + \ReLU(-x)$, hence both $h_j(x)$ and $\eta_j(x)$ can be implemented in one layer of a ReLU network.

Repeating the folding procedure for $j=1,2,\ldots,L$ results in a \emph{geometrically graded mesh}. After $L$ steps we have written everything in terms of the approximation of $f(x)$ on $[0,2^{-L}]$. Here, we can choose to approximate by zero. That choice leads to Algorithm~\ref{alg:square}.

The algorithm is easy to analyze. We use the notation $p_1(x) = h_1(x)$, $p_2(x) = h_2(p_1(x)) = h_2(h_1(x))$ and so on. We denote the approximation using $L$ folds by
\begin{equation}\label{eq:square_approx}
f_L(x) = \sum_{j=1}^L \eta_j(p_j(x)).
\end{equation}

\begin{figure}[t]
  \centering

  \begin{subfigure}[b]{0.45\textwidth}
    \centering
    \includegraphics[width=\textwidth]{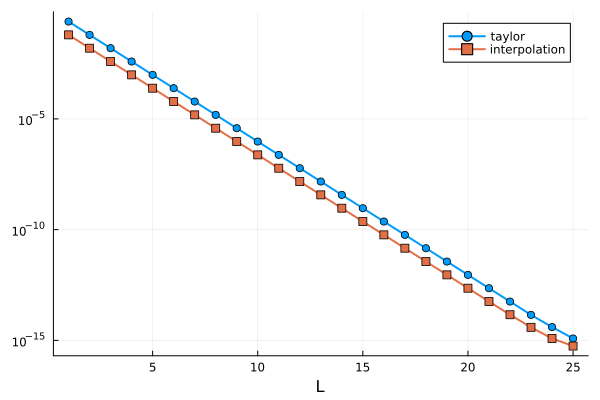}
  \end{subfigure}
  \hfill
  \begin{subfigure}[b]{0.45\textwidth}
    \centering
    \includegraphics[width=\textwidth]{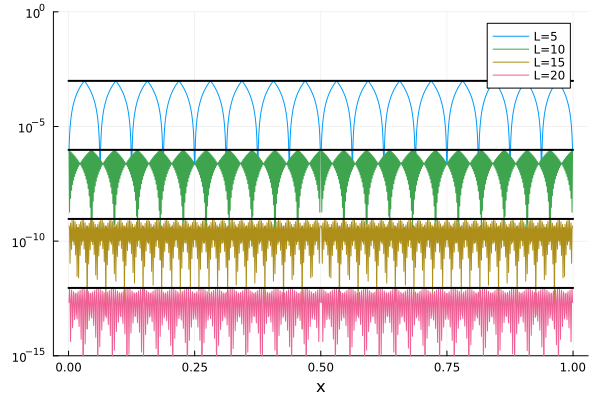}
  \end{subfigure}

  \caption{Approximation error of Algorithm~\ref{alg:square}. Left: maximum error on $[0,1]$ as a function of the level $L$, using initialization by zero or following Remark~\ref{rem:init}. Right: pointwise error on $[0,1]$ for different values of $L$. The solid black lines show the error bound $2^{-2L}$.}
  \label{fig:square}
\end{figure}

\begin{theorem}
The approximation error of~\eqref{eq:square_approx} for $x \in [0,1]$ is 
\[
x^2 - f_L(x) = p_L(x)^2.
\]
The error is bounded by $2^{-2L}$.
\end{theorem}
\begin{proof}
 The unfolding expression~\eqref{eq:square_correction} is exact. Using it $L$ times we rewrite $x^2$ as
 \[
 x^2 = p_L(x)^2 + \sum_{j=1}^L \eta_j(p_j(x)) = p_L(x)^2 + f_L(x).
 \]
For the bound we note that $p_L(x) \in [0,2^{-L}]$ for $x \in [0,1]$.
\end{proof}

\begin{remark}\label{rem:init}
Since the approximation is additive, the initialisation on $[0,2^{-L}]$ plays only a minor role. For a given initial function $u_{\textrm{init}}(x)$ we obtain the modified approximation
\[
\tilde{f}_L(x) = f_L(x) + u_{\textrm{init}}(p_L(x)).
\]
We recover precisely the approximation of Telgarsky with the choice
\[
u_{\textrm{init}}(x) = 2^{-L} \, x.
\]
Indeed, this initialization interpolates $x^2$ in both endpoints $0$ and $2^{-L}$ of $[0,2^{-L}]$ and hence, by the folding construction, in any of the dyadic points $k\, 2^{-L}$, $k=0,\ldots,2^L$, on $[0,1]$.
\end{remark}

\begin{corollary}
 The approximation error of Algorithm~\ref{alg:square} for $x \in [0,1]$ with initial approximation $u_{\textrm{init}}(x)$ for $x \in [0,2^{-L}]$ is 
\[
 x^2 - \tilde{f}_L(x) = p_L(x)^2 - u_{\textrm{init}}(p_L(x)).
\]
\end{corollary}

Amazingly, the single if-statement in Algorithm~\ref{alg:square} is responsible for exponentially many pieces in the result - such is the power of recursion. The errors are illustrated in Fig.~\ref{fig:square}. We see in the left panel that the initialisation following Telgarsky and Remark~\ref{rem:init} is slightly more accurate. In the right panel we see that the bound for the version with zero initialisation is sharp. Machine precision is reached for $L=11$ in single precision and $L=26$ in double precision floating point numbers.

\subsection{Generalizations of the folding operation}

The geometry of Algorithm~\ref{alg:square} is that it zooms in on zero using geometric grading. For univariate functions, an alternative way to zoom in on zero is to use the reflection map
\[
 \hat{h}_j(x) = | x - 2^{-j}|.
\]
This map too is continuous and piecewise linear. It has the property that it maps $[0,2^{-j+1}]$ onto $[0,2^{-j}]$, in this case by reflection with respect to $x=0$. Starting with $x\in [0,1]$ and $j=1$, repeated application of the map results in convergence to $0$. Correspondingly, relating $f(x)$ to $f(\hat{h}_j(x))$ and reasoning as above, we found approximations to $x^2$ using this map to be feasible. However, this approach does not seem to offer any advantages over the folding approach. In particular, one issue is that the error does not decrease monotonically, since $\hat{h}_j(x)$ might be larger than $x$ for some combinations of $x$ and $j$.

A more obvious generalization lies in multivariate folds. One can fold $\mathbb{R}^d$ across a hyperplane $\sum_{k=1}^d a_k x_k$. One can imagine various geometric ways in which a domain can be folded onto smaller and smaller domains, zooming in on a point such as the origin. Indeed, that is precisely the approach considered by McCane et al in~\cite{mccane2018radiallysymmetric}. They describe an approximation scheme for radially symmetric functions. These are functions of the form $f(\Vert \mathbf{x} \Vert)$, i.e., they are dependent only on the Euclidean norm of $\mathbf{x}$. Their construction in 2D folds the Euclidean space along lines at an angle of $2\pi / 2^j$ with the X-axis. This folds a disk into smaller and smaller sectors, as illustrated well by~\cite[Fig. 1(a)]{mccane2018radiallysymmetric}. Thus, eventually any vector $\mathbf{x} = (x_1,x_2)$ is mapped to a point exponentially close to $(\Vert\mathbf{x} \Vert,0)$ on the X-axis. A function $f$, which is itself realized by a ReLU network, can then be applied to the first component.

The case of multivariate folds lies closer to that of deep and dense ReLU networks than the case of univariate folds. Hence, in the quest for understanding of such networks, multivariate folds deserve more attention. However, that is not the goal of this paper. Our current focus lies with the generality of the class of algorithms that can be implemented with a ReLU network. It is clear that foldings can be generalized. As we will see, the factor that limits expressivity in our current approach is the unfolding operation. That is what we consider next.

\subsection{Generalizations of the unfolding operation}\label{ss:unfolding}

From a technical perspective there are two reasons why folding and unfolding works for $x^2$:
\begin{enumerate}
    \item The shifted function $f(a-x)$ can be related to $f(x)$ itself and vice-versa. That is what we used in~\eqref{eq:square_relate} for $a=1$ and in~\eqref{eq:square_correction} for $a = 2^{-j+1}$.
    \item In addition, the difference between $f(x)$ and $f(a-x)$ is a linear function of $x$.
\end{enumerate}
The second condition is much more restrictive than the first. The major aim in the remainder of this paper is to lift this restriction.

Several functions can be related to shifts of itself or of related functions. For example, we have the trigonometric sum identities
\begin{align}\label{eq:sincos_shift}
 \cos(a-x) &= \cos a \cos x + \sin a \sin b \\
 \sin(a-x) &= \sin a \cos x - \cos a \sin x.\nonumber
\end{align}
These identities suggest that $\cos$ and $\sin$ might be approximated simultaneously in a geometric folding approach. That would be an interesting new special case, since the ability of a neural network to approximate oscillatory functions plays a key role in the current proofs of expressiveness of deep networks~\cite{telgarsky2015representation} and in their ability to generate a basis~\cite{daubechies2022nonlinear,schneider2024rieszbasis}.
Moreover, approximation of trigonometric functions implies the possibility of neural networks to accurately represent (sparse) Fourier series.

\begin{lstlisting}[float=t,caption={Recursive approximation to $e^x$ and $e^{-x}$ on the interval {$\left[ 0,1\right]$}.}, label={alg:exp}]
function exp_approx(x, j = 1, L = 20)
    if j == L+1
        return [1.0 + x; 1.0 - x]
    end
    pivot = 1/2^j
    z = exp_approx(pivot - abs(x-pivot), j+1, L)
    if x > pivot
        return [exp(2*pivot)*z[2]; exp(-2*pivot)*z[1]]
    else
        return z
    end
end
\end{lstlisting}

\begin{lstlisting}[float=t,caption={Recursive approximation to $\cos(x)$ and $\sin(x)$ on the interval {$\left[0,\pi\right]$}.}, label={alg:sincos}]
function sincos_approx(x, j = 1, L = 10)
    if j == L+1
        return [1.0; x]
    end
    pivot = 1/2^j * pi
    z = sincos_approx(pivot - abs(x-pivot), j+1, L)
    if x > pivot
        return [cos(2*pivot) sin(2*pivot); sin(2*pivot) -cos(2*pivot)] * z
    else
        return z
    end
end
\end{lstlisting}

Another case of known shift identities is:
\begin{align}\label{eq:exp_shift}
 e^{a -x} &= e^a e^{-x} \\
 e^{-(a-x)} &= e^{-a} e^x.\nonumber
\end{align}
Again, there are two functions $f$ and $g$ and the shift identities lead to expressions for the folds $f(h_j(x))$ and $g(h_j(x))$ in terms of $f(x)$ and $g(x)$. To find the unfolding operation, these expressions have to be inverted. For the exponential functions we readily find
\begin{align}\label{eq:exp_unfold}
 e^x &= e^{a} e^{-(a-x)} \\
 e^{-x} &= e^{-a} e^{a-x}.\nonumber
\end{align}
On the interval $[0,2^{-L}]$ we can approximate $e^x$ and $e^{-x}$ by their first order Taylor series $1+x$ and $1-x$ respectively. This leads to Algorithm~\ref{alg:exp}. Recall that the intended value of $a$ is twice the pivot.

A similar algorithm can be made for the trigonometric functions by inverting~\eqref{eq:sincos_shift} and using the Taylor series approximations $\cos(x) \approx 1$ and $\sin x \approx x$ on the smallest interval. This leads to Algorithm~\ref{alg:sincos}. Note that we have opted here to approximate the functions on $\left[0,\pi\right]$, which causes the pivots to be $\pi$ times larger. We return to this approximation in \S\ref{ss:trig} later on.

\subsection{Structural versus analytical continuity}\label{ss:continuity}

Unfortunately, neither of the two algorithms for exponentials and trigonometric functions can be realized as a neural network with our current set of techniques. We will develop new techniques to do so, but first we aim to understand exactly what is the problem.

The structure of both algorithms is the same as that of Algorithm~\ref{alg:square}. However, the if-then-else statement is more general. When unfolding $x^2$ the correction term~\eqref{eq:square_correction} was continuous and piecewise linear: translation into a ReLU formulation was immediate. Here, it is clear that the result of the algorithm should be a continuous function, since we know which functions we are approximating, but continuity does not follow from the implementation itself. It is not structural, it is only analytical.

Consider the case of trigonometric functions. At step $j$ in the algorithm, the inputs are $x$ and the vector $z$, which has two components given by
\[
z =\left[\begin{array}{cc} z_1 \\ z_2
\end{array}\right] = \left[\begin{array}{cc} \cos(h_j(x)) \\ \sin(h_j(x))
\end{array}\right].
\]
There are two outputs, $\cos(x)$ and $\sin(x)$. For $\cos(x)$, we aim to realize an expression of the form
\begin{verbatim}
    if x > pivot
        A * z1 + B * z2
    else
        C * z1 + D * z2
    end
\end{verbatim}
with $A = \cos(2^{-j+1} \pi)$, $B=\sin(2^{-j+1}\pi)$, $C=1$ and $D=0$.

Neural networks can make linear combinations of inputs, but these are fixed and unconditional. The coefficients of a linear combination are stored in the form of weights in a weight matrix. However, in the expression above the coefficients differ depending on $x$: we either use $A$ and $B$, or $C$ and $D$. The problem arises because the statement itself is not continuous in $x$. Of course we know in this case that if $x$ equals the pivot and $z_1$ and $z_2$ are as above, it is true that
\[
A z_1 + B z_2 = C z_1 + D z_2.
\]
Yet, this equality is not true for any combination of $z_1$ and $z_2$.

One way to proceed is to relax the problem. Indeed, it is sufficient to construct a neural network that agrees with the conditional expression in the algorithm whenever the values in $z_1$ and $z_2$ actually correspond to the folds of $\cos$ and $\sin$. Given other inputs, the network may produce output that is unrelated to trigonometric functions. That is not a problem since, if all goes well in the preceding layers, it will never happen in practice. We will consider conditional expressions formally in~\S\ref{s:branching}, and then return to concrete examples in~\S\ref{s:examples}.

\section{Conditional branching in a ReLU network}\label{s:branching}

Consider a layer of a neural network that has three inputs. Assuming that the input of the network itself is $x$, each input is a function of $x$. We denote the three inputs by $a(x)$, $b(x)$ and $c(x)$. The question we address in this section is to what extent we can realize the conditional statement
\begin{verbatim}
    if a(x) >= 0 then
        b(x)
    else
        c(x)
    end
\end{verbatim}

We have already noted that realizing this expression with ReLU networks is impossible unless it is continuous. The branches $b(x)$ and $c(x)$ should agree for any critical value of $x$ for which $a(x)=0$.

\subsection{A few simplifying assumptions}

We will for simplicity assume that $x$ is scalar. If the network has high-dimensional input, we can still consider branching based on one element of the input. Furthermore, we will assume that $x$ is restricted to an interval $[x_0,x_1]$. Thus, the continuity condition can be stated as:
\begin{equation}\label{eq:continuity}
\forall x \in [x_0,x_1]: a(x) = 0 \Rightarrow b(x) = c(x).
\end{equation}

We will make the further assumption that there is only one critical value $x^*$ in the interval $[x_0,x_1]$:
\[
\forall x \in [x_0,x_1]: a(x) = 0 \Rightarrow x = x^*.
\]
If there are several isolated critical values of $x$, one can always subdivide the interval $[x_0,x_1]$. With this assumption, \eqref{eq:continuity} becomes
\[
b(x^*) = c(x^*).
\]

Finally, we will assume that $a(x)$ changes sign at $x^*$. Thus, it is either positive on one side and negative on the other, or vice-versa. This is not a real restriction, since if $a(x)$ is positive on both sides the condition $a(x) \geq 0$ is always satisfied and the result is simply $b(x)$. Similarly, if $a(x)$ is negative the outcome is $c(x)$.

\subsection{A special case and an assumption about signs}\label{ss:concatenation}

We start with a simpler setting involving just two input functions along with $x$ itself. We consider the problem of concatenating two functions with different supports in $x$. Let's say one function $b(x)$ is defined for $x\geq0$ and another function $c(x)$ for $x < 0$. Their concatenation corresponds to the expression
\begin{verbatim}
    if x >= 0
        b(x)
    else
        c(x)
    end
\end{verbatim}
with $b(0)=c(0)$ for continuity.

One way to realize the outcome is via a linear combination of the two input functions with fixed weights. To that end we aim to extend the functions outside their support in a suitable way. Define
\begin{equation}\label{eq:bc_extension}
\tilde{b}(x) = \left\{ \begin{array}{cc}
0, & x < 0 \\
b(x)-b(0), & x \geq 0
\end{array}\right. \quad \mbox{and} \quad \tilde{c}(x) = \left\{ \begin{array}{cc}
c(x)-c(0), & x < 0 \\
0, & x \geq 0.
\end{array}\right.
\end{equation}
In this case the outcome of the conditional statement is the sum
\begin{equation}\label{eq:stitching}
z(x) = b(0) + \tilde{b}(x) + \tilde{c}(x).
\end{equation}
Only one of the two functions $\tilde{b}(x)$ and $\tilde{c}(x)$ is non-zero for any given value of $x$. This corresponds exactly to one of the two branches being taken in the conditional statement.

\begin{remark}
    Note that in~\eqref{eq:stitching} we have chosen to add the constant $b(0)$. We could also have chosen $c(0)$. In so doing, we assume that the continuity condition $b(0)=c(0)$ holds, but it is not enforced. As a consequence, \eqref{eq:continuity} agrees with the conditional statement for all $x$ if $b(0)=c(0)$ holds, but it may also produce a different outcome if the condition is violated. The network output will be continuous even if the concatenation of $b$ and $c$ is not.
\end{remark}

We can realize $\tilde{b}(x)$ and $\tilde{c}(x)$ with ReLU by making further analytical assumptions. Say $b(x)$ and $c(x)$ are actually defined for all $x$. The function $\hat{b}(x) = b(x)-b(0)$ vanishes at $x=0$. If $\hat{b}(x)$ changes sign at $x=0$ and is positive for positive $x$ then we can realize the extension as
\[
\tilde{b}(x) = \ReLU(\hat{b}(x)) = \ReLU(b(x)-b(0)).
\]
If $\hat{b}(x)$ changes sign and is negative for $x \geq 0$ then we have
\[
\tilde{b}(x) = \ReLU(-\hat{b}(x)) = \ReLU(-b(x)+b(0)).
\]
We can make a similar assumption for $c(x)$ which we formalize in Assumption~\ref{as:assumption1}.
\begin{assumption}[sign-reversal]\label{as:assumption1}
 The two functions $b$ and $c$ are such that $b(0)=c(0)$, and $b(x)-b(0)$ and $c(x)-c(0)$ both change signs at $x=0$.
\end{assumption}

In case the functions satisfy Assumption~\ref{as:assumption1} and both are positive for $x \geq 0$ we conclude that~\eqref{eq:stitching} is equivalent to
\begin{equation}\label{eq:stitching_relu}
 z(x) = b(0) + \ReLU(b(x)-b(0)) + \ReLU(c(x)-c(0)).
\end{equation}
Similar expressions can be found depending on the signs of $\hat{b}$ and $\hat{c}$ on either side of $x=0$.

Unfortunately, Assumption~\ref{as:assumption1} is restrictive for the class of problems we are considering. It is an entirely valid scenario in which $\hat{b}$ or $\hat{c}$ vanish at $0$ without changing sign. In that case, we need a more elaborate technique.

\subsection{The general case and an assumption about relative sizes}\label{ss:generalcase}

We go back to the general conditional statement involving three inputs $a(x)$, $b(x)$ and $c(x)$. One property of ReLU is that its output can only equal its input or zero. Thus, $\ReLU(a(x))$ can never produce $b(x)$. However, we have seen in~\S\ref{s:sorting} that it is possible to interchange inputs based on their relative sizes, since that operation is continuous. Thus, one can realize one of two possible different outcomes by sorting two inputs. We are led to consider the relative sizes of the inputs.

We develop this idea further. Define as in~\S\ref{ss:concatenation} the normalized functions
\[
\hat{b}(x) = b(x) - b(x^*) \quad \mbox{and} \quad \hat{c}(x) = c(x)-c(x^*).
\]
They are normalized in the sense that $a(x^*)=\hat{b}(x^*)=\hat{c}(x^*)=0$. We can still think of the general conditional expression as concatenating two functions, similar to~\S\ref{ss:concatenation}, since we have assumed that $a(x)$ changes sign on $[x_0,x_1]$ exactly one. Thus, for $x \in [x_0,x_1]$, by defining the extensions
\begin{equation}\label{eq:bc_extension_general}
\tilde{b}(x) = \left\{ \begin{array}{cc}
0, & a(x) < 0 \\
\hat{b}(x), & a(x) \geq 0
\end{array}\right. \quad \mbox{and} \quad \tilde{c}(x) = \left\{ \begin{array}{cc}
\hat{c}(x), & a(x) < 0 \\
0, & a(x) \geq 0,
\end{array}\right.
\end{equation}
we can write the outcome as
\begin{equation}\label{eq:stitching_general}
z(x) = b(0) + \tilde{b}(x) + \tilde{c}(x).
\end{equation}
This is almost the same as~\eqref{eq:stitching}. However, the condition $a(x) \geq 0$ corresponds to $x \geq x^*$ only if $a(x)$ is positive for $x > x^*$. It corresponds to $x < x^*$ in the other case, in which case compared to~\S\ref{ss:concatenation} we also stitch the functions together in reverse order.

In order to make statements about the sizes of $\hat{b}(x)$ and $\hat{c}(x)$ relative to $a(x)$, we define a notion of crossing functions as follows.

\begin{definition}\label{def:crossing}
 We say that two continuous functions $f$ and $g$ are crossing at $x^*$ if $f(x)-g(x)$ changes sign at $x^*$.
\end{definition}
This definition is illustrated in Fig.~\ref{fig:crossing} in the setting we like to use. The notion of crossing functions allows a weaker assumption than Assumption~\ref{as:assumption1} before. As shown in the figure, it is not necessary for $\hat{b}(x)$ to change sign at $x^*$. In addition, we can allow for some more flexibility by introducing a constant. This constant will have to be determined analytically later on.

\begin{figure}
    \begin{center}
    \includegraphics[width=0.8\linewidth]{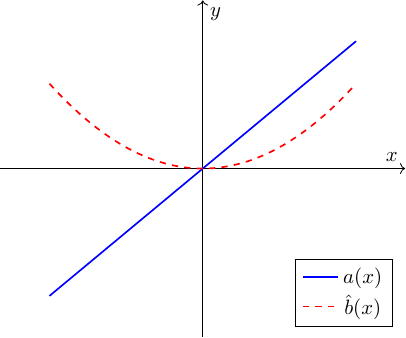}
    \end{center}
\caption{Illustration of two crossing functions following Definition~\ref{def:crossing}. The function $a$ changes sign at $x^*=0$ and $\hat{b}$ does not, but they are such that $a(x) > \hat{b}(x)$ for $x>0$ and $a(x)<\hat{b}(x)$ for $x<0$.}\label{fig:crossing}
\end{figure}

\begin{assumption}[crossing assumption]\label{as:assumption2}
There exists a constant $\beta > 0$ such that $\hat{b}(x)$ and $\beta a(x)$ are crossing functions, and a constant $\gamma  > 0$ such that $\hat{c}(x)$ and $\gamma  a(x)$ are also crossing functions.
\end{assumption}

Assumption~\ref{as:assumption2} allows us to reformulate the conditional expression as a fixed linear combination of ReLU function calls. To see how, let us consider the case where $a(x) > 0$ for $x > x^*$, and where $\beta a(x)-\hat{b}$ and $\gamma  a(x)-\hat{c}(x)$ are positive for $x > x^*$ as well. One such case is shown in Fig.~\ref{fig:crossing} for $x^*=0$ and $\beta=1$. In Table~\ref{tab:logical} we show the outputs of ReLU for several different inputs. The question becomes whether we can realize the desired outcome, shown in the last row, by taking a linear combination of the other rows with fixed constants.

\begin{table}[ht]
\centering
\begin{tabular}{c|c|c}
 & $x < x^*$ & $x > x^*$\\
\hline
$\ReLU\left(a(x)\right)$  & $0$ &  $a(x)$  \\
$\ReLU\left(-a(x)\right)$  & $-a(x)$ &  $0$  \\
$\ReLU\left(\beta a(x) - \hat{b}(x)\right)$  & $0$ &  $\beta a(x) - \hat{b}(x)$  \\
$\ReLU\left(-\beta a(x) + \hat{b}(x)\right)$  & $-\beta a(x) + \hat{b}(x)$ &  $0$ \\
$\ReLU\left(\gamma a(x) - \hat{c}(x)\right)$  & $0$ &  $\gamma a(x) - \hat{c}(x)$  \\
$\ReLU\left(-\gamma a(x) + \hat{c}(x)\right)$  & $-\gamma a(x) + \hat{c}(x)$ &  $0$ \\ \hline 
\texttt{if} $a(x) > 0$  \texttt{then} $b(x)$ \texttt{else} $c(x)$ & $c(x)$ & $b(x)$
\end{tabular}
\caption{Logical table that shows the outcome of several ReLU expressions for $x < x^*$ and $x>x^*$ respectively, in the case where $a$, $\beta a-\hat{b}$ and $\gamma a-\hat{c}$ are positive for $x>x^*$. The last line shows the desired outcome of the conditional expression.}
\label{tab:logical}
\end{table}

To that end we first note that the extended function $\tilde{b}(x)$ as defined by~\eqref{eq:bc_extension_general} can be written as
\[
 \tilde{b}(x) = - \ReLU\left(\beta a(x) - \hat{b}(x)\right) + \beta  \, \ReLU\left(a(x)\right).
\]
Indeed, for $x < x^*$ this expression results in $0$, while for $x > x^*$ we obtain $\hat{b}(x) = b(x)-b(0)$. The corresponding expression for $\tilde{c}$ is
\[
 \tilde{c}(x) = \ReLU\left(-\beta a(x) + \hat{c}(x)\right) - \beta \ReLU\left(-a(x)\right).
\]
This expression evaluates to $0$ for $x > x^*$, and to $\hat{c}(x)$ for $x < x^*$. Taken together, we find that
\begin{align}\label{eq:generalcase_result}
    z(x) &= b(0) + \tilde{b}(x) + \tilde{c}(x) \nonumber \\ 
    &= b(0) - \ReLU\left(\beta a(x) - \hat{b}(x)\right) + \beta  \, \ReLU\left(a(x)\right) \\
    & \quad + \ReLU\left(-\gamma a(x) + \hat{c}(x)\right) - \gamma  \, \ReLU\left(-a(x)\right). \nonumber
\end{align}
This expression results in $b(x)$ for $x>x^*$ and in $c(x)$ for $x < x^*$. Similar expressions can be found for the other cases, depending on the signs of $\beta  a(x) - \tilde{b}(x)$ and $\gamma  a(x) - \tilde{c}(x)$ for $x > x^*$.

Subject to Assumption~\ref{as:assumption2}, we can realize any conditional statement in terms of ReLU. The main practical issue in applying~\eqref{eq:generalcase_result} lies in the explicit identification of the constants $\beta$ and $\gamma$. This step requires analytical knowledge. However, the assumption that such constants exist is in itself not very restrictive. Since we assume a bounded interval for $x$, it suffices for the functions involved to be Lipschitz continuous on that interval. Assumption~\ref{as:assumption2} is quite generic, and one could even imagine deducing the parameters automatically based on function samples. In contrast, Assumption~\ref{as:assumption1} poses a genuine restriction in practice.

\section{Algorithmic neural networks through examples}\label{s:examples}

We use the methodology of~\S\ref{s:branching} to construct deep neural networks for a number of special functions. The networks in this section have a similar structure:
\begin{itemize}
\item $L$ layers apply a folding operation, 
\item $L$ subsequent layers apply an unfolding operation.
\end{itemize}
One difference compared to the case of $x^2$ is that we have to decide between branches based on the value of $x$ compared to the current pivot. This means that the value of $p_j(x)$, which is computed in the $j$-th folding operation, should remain available for the $j$-th unfolding operation. This can be achieved by skipping connections, or as before by using two neurons for each quantity to be maintained across layers based on the identity $f(x)=\ReLU(x) - \ReLU(-x)$. Since there are $L$ such quantities, the networks have a width of approximately $2L$.

\subsection{The exponential function}

\begin{figure}[t]
  \centering

  \begin{subfigure}[b]{0.45\textwidth}
    \centering
    \includegraphics[width=\textwidth]{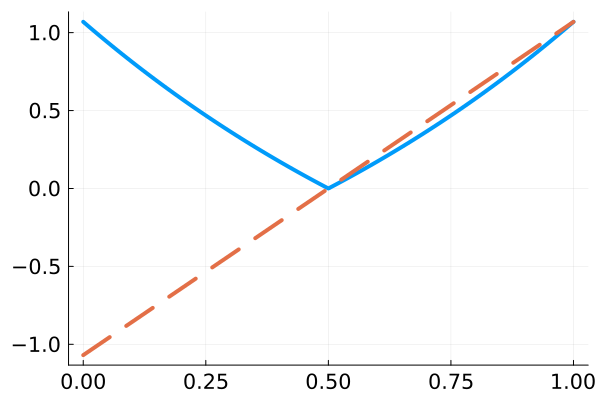}
    \caption{$\hat{b}_1(x)$ and $\beta_1 a(x)$}
  \end{subfigure}
  \hfill
  \begin{subfigure}[b]{0.45\textwidth}
    \centering
    \includegraphics[width=\textwidth]{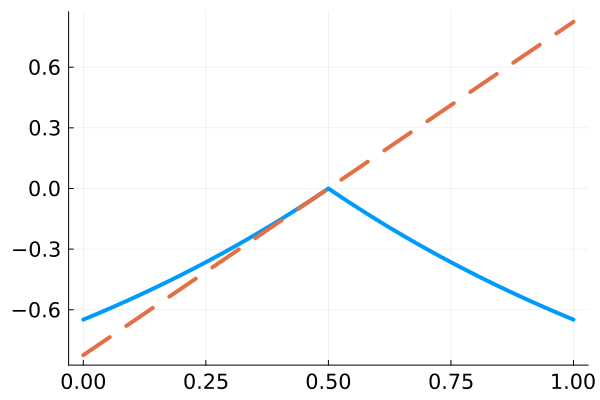}
    \caption{$\hat{c}_1(x)$ and $\gamma_1 a(x)$}
  \end{subfigure}

  \begin{subfigure}[b]{0.45\textwidth}
    \centering
    \includegraphics[width=\textwidth]{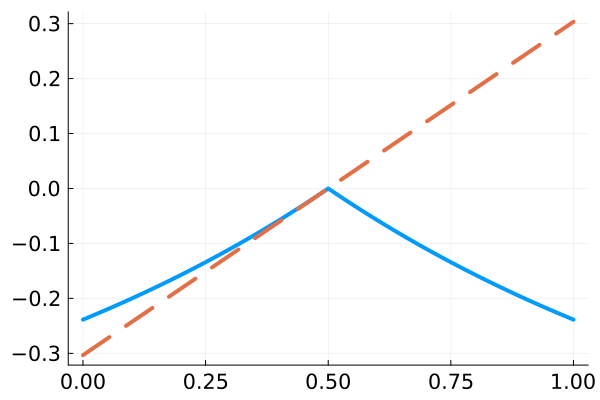}
    \caption{$\hat{b}_2(x)$ and $\beta_2 a(x)$}
  \end{subfigure}
  \hfill
  \begin{subfigure}[b]{0.45\textwidth}
    \centering
    \includegraphics[width=\textwidth]{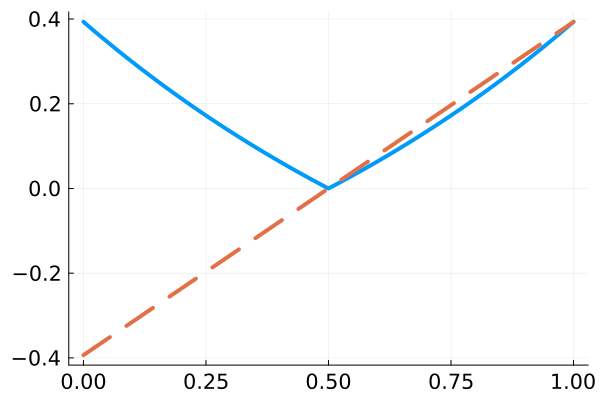}
    \caption{$\hat{c}_2(x)$ and $\gamma_2 a(x)$}
  \end{subfigure}

  \caption{Illustration of the quantities in the conditional expression for Algorithm~\ref{alg:exp}. In each case, the functions $\hat{b}(x)$ or $\hat{c}(x)$ (solid line) are greater than a multiple of $a(x)$ (dashed line) on one side of the pivot, and smaller on the other side. This property is used to realize $\tilde{c}(x)$ and $\tilde{b}(x)$ using ReLU.}\label{fig:conditional_exp}
\end{figure}

We revisit Algorithm~\ref{alg:exp}. Note that we can write the absolute value function as
\[
|x| = \ReLU(x) + \ReLU(-x).
\]
The application of $h_j$ involves one absolute value function. Denoting the pivot by $s_j = 2^{-j}$, we obtain
\[
h_j(x) = s_j - |x - s_j| = s_j - \ReLU(x - s_j) - \ReLU(s_j - x).
\]
Furthermore, let $z(x) =\left[\begin{array}{cc} z_1(x) \\ z_2(x)
\end{array}\right]$ denote the outcome of the recursive function call. By induction, we can assume that
\[
z_1(x) = e^{h_j(x)} = e^{s_j - | x - s_j|} \quad \mbox{and} \quad z_2(x) = e^{-h_j(x)} = e^{-s_j + | x - s_j|}.
\]

Our goal is to realize the conditional expression at step $j$ in terms of ReLU. The statement has a vector-valued result with two entries. Thus, there are really two conditional statements. Identifying with the form of~\S\ref{s:branching}, the first conditional expression involves the functions
\begin{align*}
a(x) &= x-s_j, \\
b_1(x) &= e^{2 s_j} z_2(x),\\
c_1(x) &= z_1(x).
\end{align*}
We have $x^*=s_j$ and $b_1(x^*)=c_1(x^*)=e^{s_j}$. The second conditional expression has the same condition $a(x) > 0$ and branches given by
\begin{align*}
b_2(x) &= e^{-2s_j} z_1(x),\\
c_2(x) &= z_2(x).
\end{align*}

Consider the first case. By definition we have $\hat{b}_1(x) = b(x) - e^{s_j}$ and $\hat{c}_1(x) = c(x)-e^{s_j}$. These functions are shown in the top row of Fig.~\ref{fig:conditional_exp}. In order to realize the statement with ReLU's, we need to determine constants $\beta_1$ and $\gamma_1$ such that Assumption~\ref{as:assumption2} holds. For $\tilde{c}_1(x)$, shown in Fig.~\ref{fig:conditional_exp}(b), we can use the tangent line of the exponential at the pivot, which leads to
\[
\gamma_1 = e^{s_j}.
\]
For $\tilde{b}_1(x)$, shown in Fig.~\ref{fig:conditional_exp}(a), we have to take into account that the exponential achieves its maximum in the point $x=1$. By direct computation, we find the slope of the line connecting the value at $x=\frac12$ to the value at $x=1$ to be
\[
\beta_1 = \frac{e^{2s_j} - e^{s_j}}{s_j}.
\]
For the second case, we find by similar reasoning the values
\[
\beta_2 = e^{-s_j} \quad \mbox{and} \quad \gamma_2 = \frac{1-e^{-s_j}}{s_j}.
\]
We can now use~\eqref{eq:generalcase_result} two times to realize the algorithm.

\begin{figure}[t]
  \centering

  \begin{subfigure}[b]{0.45\textwidth}
    \centering
    \includegraphics[width=\textwidth]{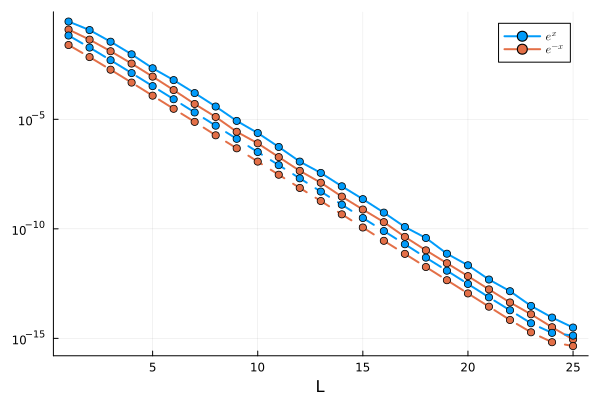}
  \end{subfigure}
  \hfill
  \begin{subfigure}[b]{0.45\textwidth}
    \centering
    \includegraphics[width=\textwidth]{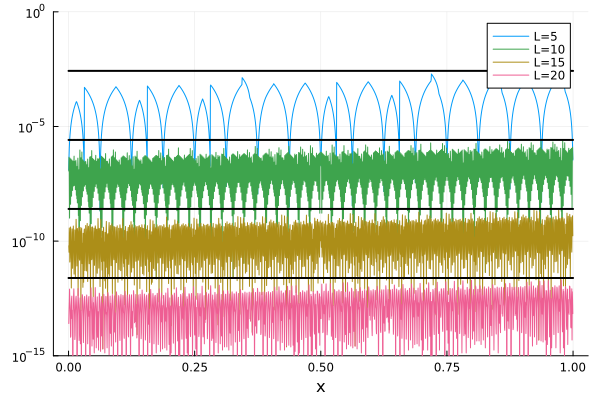}
  \end{subfigure}

  \caption{Approximation error of Algorithm~\ref{alg:exp}. Left: maximum error on $[0,1]$ for increasing $L$ for the functions $e^x$ and $e^{-x}$ using Taylor series initialization (solid) or interpolation (dashed). Right: pointwise error on $[0,1]$ for different values of $L$. The solid black lines show the values $e^1 \,2^{-2L}$.}
  \label{fig:exp_error}
\end{figure}

For completeness, we note that the analysis of the algorithm remains straightforward. The main source of error is the approximation on the smallest interval $[0,2^{-L}]$. Approximation by the first order Taylor series carries an error that is quadratic. Hence, this error scales as $2^{-2L}$. Instead of a Taylor series one can also choose the initial approximation to be interpolating in both endpoints of $[0,2^{-L}]$. Explicitly, this is achieved by the linear expression
\[
f(x) \approx f(0) + \left[f(2^{-L}) - f(0)\right] \frac{x}{2^{-L}}, \qquad x \in [0,2^{-L}].
\]
The associated error is also quadratic. Furthermore, in this case the result of Algorithm~\ref{alg:exp} and of other similar algorithms interpolates $f(x)$ at all dyadic points, as was the case for $x^2$.

The errors in the approximation of the exponential functions are shown in Fig.~\ref{fig:exp_error}. We see that, again, the interpolating variant is slightly more accurate. Both variants achieve a maximum error that scales as $2^{-2L}$, as shown in the right panel.

\subsection{Trigonometric functions}\label{ss:trig}

\begin{figure}[t]
  \centering

  \begin{subfigure}[b]{0.45\textwidth}
    \centering
    \includegraphics[width=\textwidth]{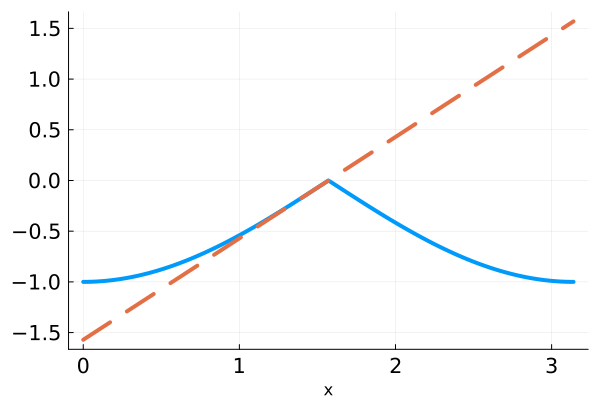}
    \caption{$\hat{b}_1(x)$}
  \end{subfigure}
  \hfill
  \begin{subfigure}[b]{0.45\textwidth}
    \centering
    \includegraphics[width=\textwidth]{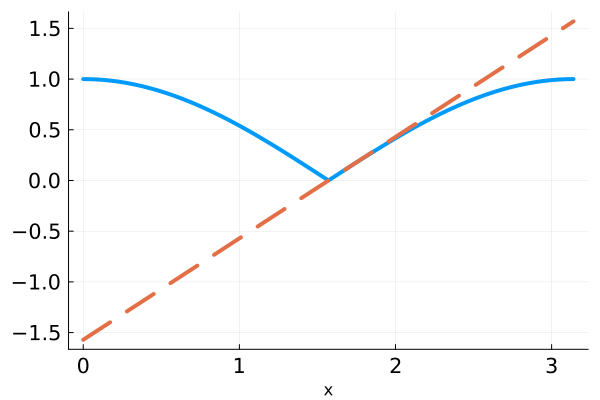}
    \caption{$\hat{c}_1(x)$}
  \end{subfigure}

  \begin{subfigure}[b]{0.45\textwidth}
    \centering
    \includegraphics[width=\textwidth]{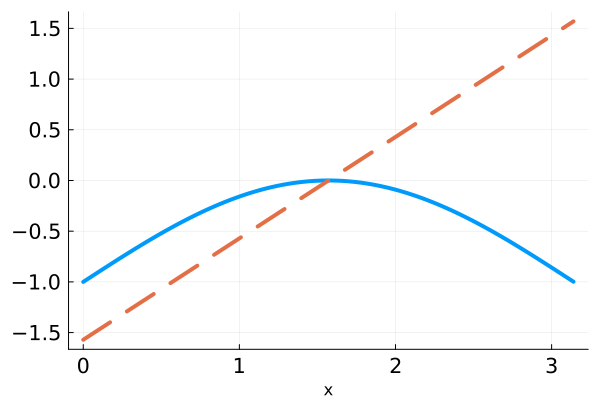}
    \caption{$\hat{b}_2(x)$}
  \end{subfigure}
  \hfill
  \begin{subfigure}[b]{0.45\textwidth}
    \centering
    \includegraphics[width=\textwidth]{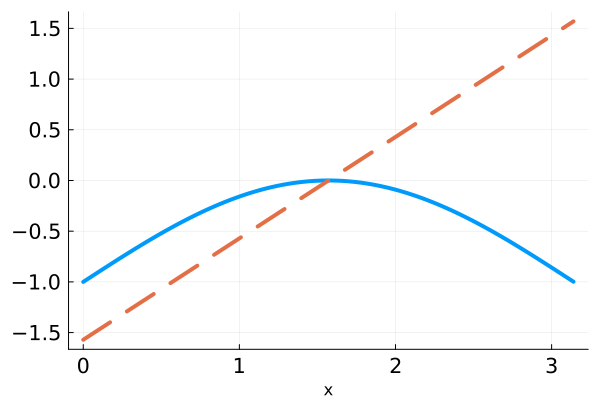}
    \caption{$\hat{c}_2(x)$}
  \end{subfigure}

  \caption{This figure is the analogue of Fig.~\ref{fig:conditional_exp} for the case of Algorithm~\ref{alg:sincos}. In each case, the functions $\hat{b}(x)$ and $\hat{c}(x)$ (solid) cross with the line $x$ (dashed). This implies that we can choose the constants $\beta_1=\gamma_1=\beta_2=\gamma_2=1$. The functions are shown for level $j=1$ on the interval $[0,\pi]$.}\label{fig:conditional_trig}
\end{figure}

The implementation of Algorithm~\ref{alg:sincos} in terms of ReLU is entirely analogous to that of Algorithm~\ref{alg:exp}. Here we can choose the constants $\beta_1=\gamma_1=\beta_2=\gamma_2=1$, as illustrated in Fig.~\ref{fig:conditional_trig} for the first level.

The approximation errors are shown in Fig.~\ref{fig:trig_error}. The errors for $\cos(x)$ and $\sin(x)$ are almost identical, with interpolating initialization on $[0,\pi 2^{-L}]$ again presenting a small edge. In the right panel, we see that the approximation of $\sin(x)$ is more accurate near the endpoints and in the middle. That is because the sine function vanishes and in addition obeys odd symmetry: the local approximation error is cubic rather than quadratic.

\begin{figure}[t]
  \centering

  \begin{subfigure}[b]{0.45\textwidth}
    \centering
    \includegraphics[width=\textwidth]{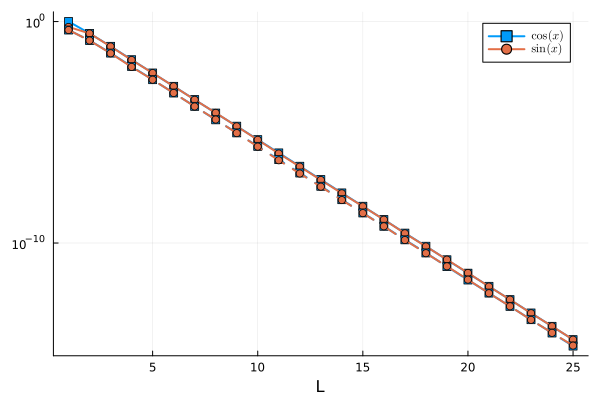}
  \end{subfigure}
  \hfill
  \begin{subfigure}[b]{0.45\textwidth}
    \centering
    \includegraphics[width=\textwidth]{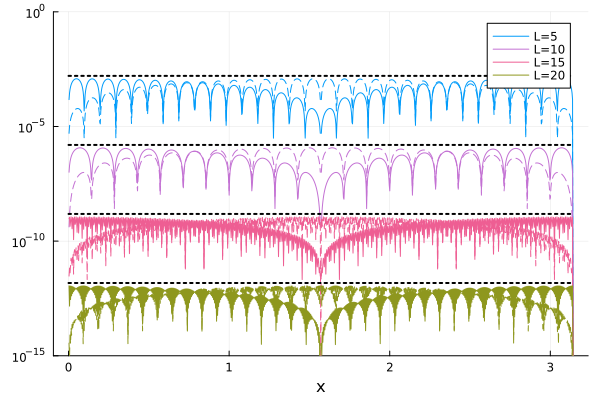}
  \end{subfigure}

  \caption{Approximation error of Algorithm~\ref{alg:sincos}. Left: maximum error on $[0,\pi]$ for increasing $L$ for the functions $\cos(x)$ and $\sin(x)$ using Taylor series initialization (solid) or interpolation (dashed). Right: pointwise error on $[0,\pi]$ for different values of $L$ for $\cos(x)$ (solid) and $\sin(x)$ (dashed), using interpolation. The dotted black lines show the values $2^{-2L}\pi^2/6$.}
  \label{fig:trig_error}
\end{figure}

\subsubsection{Increasing the approximation domain via symmetry}

The approximations were defined on the interval $[0,\pi]$ which, due to symmetry, is sufficient to determine $\cos(x)$ and $\sin(x)$ everywhere. Consider the cosine function, which has even symmetry with respect to the points $x=0$ and $x=\pi$. We can use the sawtooth function~\eqref{eq:sawtooth} to exploit that symmetry. Given an approximation $f_L(x) \approx \cos(x)$, for $x \in [0,\pi]$, we can define a periodic approximation
\[
 f_{L,s}(x) = f_L\left(\pi g_s\left( \frac{x}{\pi 2^s} \right) \right), \qquad x \in [0,\pi 2^s].
\]
The sawtooth function, appropriately scaled as above, folds the large interval $[0,\pi 2^s]$ onto $[0,\pi]$ in a way that is compatible with the symmetries of the cosine function.

In other words, by adding $s$ preprocessing layers to the neural network approximation of $\cos(x)$, we can increase the size of the approximation domain by an exponential factor.

The sine function satisfies odd symmetry with respect to the points $0$ and $\pi$. Thus, each fold also comes with a multiplication by $-1$. This is not a continuous operation. One can start from the sine function on the interval $\left[-\frac{\pi}{2},\frac{\pi}{2}\right]$ instead, such that it has even symmetries, or one can simply use the property $\sin(x) = \cos\left(\frac{\pi}{2}-x\right)$ and approximate the cosine.

\subsection{A polynomial basis: monomials}\label{ss:monomials}

\begin{lstlisting}[float=t,caption={Recursive approximation to $x^k$ on the interval {$\left[ 0,1\right]$} for {$k=2,\ldots,d$}.}, label={alg:monomials}]
function monomials_approx(x, d, j = 1, L = 20)
    if j == L+1
        xL = 1/2^L
        return [xL^k * x/xL for k in 0:d]    # interpolating approximation
    end
    pivot = 1/2^j
    z = monomials_approx(pivot - abs(x-pivot), d, j+1, L)
    if x > pivot
        u = zeros(degree+1)
        u[1] = 1
        u[2] = x
        for k in 2:degree
            u[k+1] = (-1)^k*z[k+1]
            for l in 0:k-1
                u[k+1] = u[k+1] - binomial(k, l)*(-1)^(k+l)*2^(k-l)*pivot^(k-l) * u[l+1]
            end
        end
        return u
    else
        return z
    end
end
\end{lstlisting}

Polynomials of higher degree can be realized in terms of a neural network approximation of $x^2$ and one of the multiplication operator. One multiplies $x^2$ by $x$ a number of times. This construction is frequently used in literature~\cite{grohs2022expressivity,opschoor2019highorderfem}. However, it is not the only or most efficient way.

A network with smaller depth results from studying the folding operation. In the case where $x$ is larger than the pivot $s_j$, we aim to relate $f(x)$ to $f(2s_j-x)$. We note for the case of monomials that
\[
 h_j(x)^k = (2s_j-x)^k = \sum_{l=0}^k \binom{k}{l} (2s_j)^{k-l} (-x)^l, \qquad x \geq s_j.
\]
This relation suggests that we should approximate the functions $\{ x^l \}_{l=0}^k$ together. Solving for $x^k$ we find the unfolding operation
\begin{equation}\label{eq:monomials_unfolding}    
 x^k = (-1)^k \, h_j(s)^k - \sum_{l=0}^{k-1} \binom{k}{l}(-1)^{k+l} 2^{k-l} s^{k-l} x^l.
\end{equation}
Thus, having determined the powers of $x$ up to degree $k-1$, the relation above yields the next degree monomial. This leads to Algorithm~\ref{alg:monomials}.

The unfolding operation~\eqref{eq:monomials_unfolding} is more complicated than the ones we have seen before. However, the relation is affine. Indeed, we can express $x^k$ affinely in terms of $h_j(x)^k$ and lower degree monomials. In turn, the lower degree monomials $x^l$ can be expressed affinely in terms of $h_j(x)^l$. Thus, the map from $\{ h_j(x)^l \}_{l=0}^k$ to $\{ x^l \}_{l=0}^k$ is affine. It can be implemented in a single layer.

\begin{figure}[t]
  \centering

  \begin{subfigure}[b]{0.45\textwidth}
    \centering
    \includegraphics[width=\textwidth]{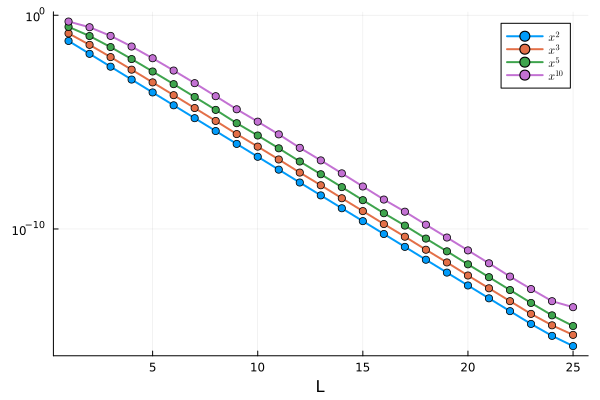}
  \end{subfigure}
  \hfill
  \begin{subfigure}[b]{0.45\textwidth}
    \centering
    \includegraphics[width=\textwidth]{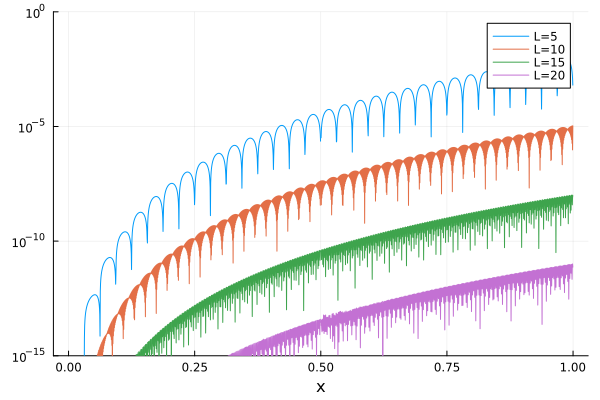}
  \end{subfigure}

  \caption{Approximation error of Algorithm~\ref{alg:monomials} for the computation of monomials up to degree $d=10$. Left: maximum error on $[0,1]$ for increasing $L$ and for a few degrees. Right: pointwise error on $[0,1]$ for different values of $L$ for $f(x) = x^{10}$.}
  \label{fig:monomials_error}
\end{figure}

The constant term $x^0$ and the first monomial $x^1$ need not be computed. There remain $d-1$ conditional expressions in Algorithm~\ref{alg:monomials}. The branches define corresponding functions $b_k(x)$ and $c_k(x)$. To satisfy Assumption~\ref{as:assumption2}, we observe that we can choose $\gamma_k = 1$ and $\beta_k = k s_j^{k-1}$, where $s_j$ is the pivot. Note that $\beta_k$ is precisely the derivative of $x^k$ at the pivot. Rapid convergence is illustrated in Fig.~\ref{fig:monomials_error}. The largest network when producing this figure corresponds to the parameters $L=25$ and polynomial degree $d=10$. That network has depth $52$ and width $68$.

\subsection{The multiplication operator}

We end with a bivariate example, the multiplication function $f(x,y)=xy$. We aim to avoid its indirect approxiation via multiple squares as expressed by~\eqref{eq:multiplication}.

There are several possibilities to fold the unit square $[0,1]^2$ onto smaller and smaller squares. The simplest choice is to refine $x$ as in the previous examples, while keeping $y$ fixed. We find for $x > s_j$ that
\[
h_j(x) y = (2s_j-x)y = 2 s_j y - xy,
\]
from which we find the unfolding relation
\[
xy = -h_j(x)y + 2 s_j y.
\]
This is a simple relation because, for fixed $y$, the multiplication function is linear in $x$. However, the approximation on $[0,2^{-L}]$ presents an issue. The linear Taylor series approximation to $xy$ is $xy$ itself, which we can not realize. Instead, we can choose to approximate $xy$ by zero. Yet, this approximation is only first order accurate. As a result, the overall scheme converges at a rate proportional to $2^{-L}$. The previous algorithms converge much more rapidly at the rate $2^{-2L}$.

Quadratic convergence can be recovered by refining both in $x$ and in $y$. We could do so in alternating order. That leads to Algorithm~\ref{alg:mul}. In this case, approximation by zero on for $(x,y) \in [0,2^{-L}] \times [0,2^{-L}]$ yields an error that scales as $2^{-2L}$.

\begin{lstlisting}[float=t,caption={Recursive approximation to $f(x,y)=xy$ on the square {$\left[ 0,1\right]^2$}.}, label={alg:mul}]
function mul_approx(x, y, j = 1, L = 20)
    if j == L+1
        return 0.0
    end
    pivot = 1/2^j
    z = mul_approx_y(pivot - abs(x-pivot), y, j, L)
    if x > pivot
        2*pivot*y - z
    else
        z
    end
end

function mul_approx_y(x, y, j, L)
    pivot = 1/2^j
    z = mul_approx(x, pivot - abs(y-pivot), j+1, L)
    if y > pivot
        2*pivot*x - z
    else
        z
    end
end
\end{lstlisting}

\begin{figure}[t]
  \centering

  \begin{subfigure}[b]{0.45\textwidth}
    \centering
    \includegraphics[width=\textwidth]{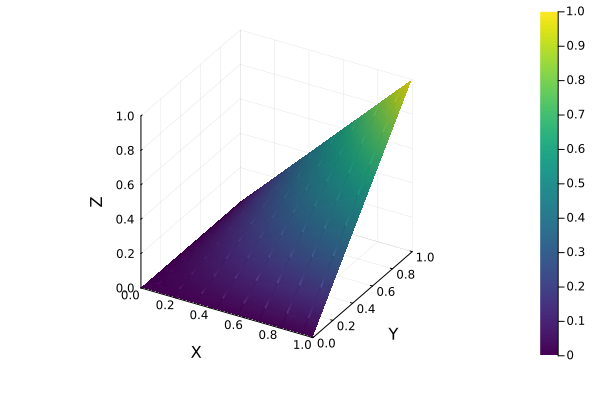}
    \caption{Surface plot of $xy$ on $[0,1]^2$}
  \end{subfigure}
  \hfill
  \begin{subfigure}[b]{0.45\textwidth}
    \centering
    \includegraphics[width=\textwidth]{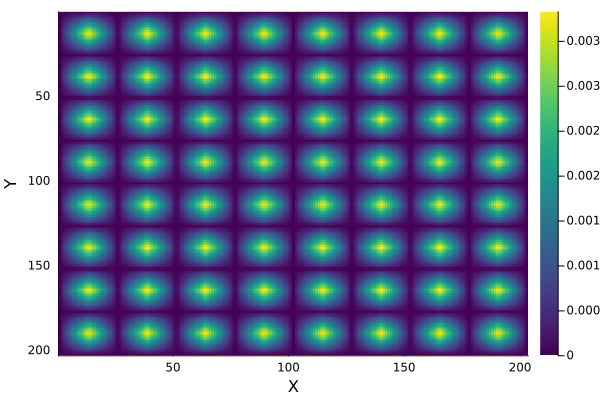}
    \caption{Approximation error for $L=4$}
  \end{subfigure}
  \caption{Illustration of the approximation to $f(x,y)=xy$ by Algorithm~\ref{alg:mul}.}
  \label{fig:mul}
\end{figure}

Some results are shown in Fig.~\ref{fig:mul}. The left panel shows the bivariate function being approximated. The right panels shows the error of Algorithm~\ref{fig:mul} for $L=4$. It clearly shows an $8 \times 8$ grid of pieces. Each of those pieces corresponds to a planar piece approximating a part of the surface shown in the left panel. The convergence behaviour for larger $L$ is similar to that of the previous examples.

\section{Discussion and concluding remarks}

\subsection{What is the scope of neural network algorithms?}

We have offered no concrete answer to the question which classes of algorithms can be realized with a neural network. We have focused on conditional branching in~\S\ref{s:branching}. Based on the examples in this paper, we can make a few additional observations:
\begin{itemize}
    \item in our examples one function call corresponds to one layer
    \item the recursion depth is fixed by the size of the network and bounded by the number of layers
    \item local variables of fixed size can be implemented by adding neurons to a layer, such as the two-element vector $z$ in Algorithm~\ref{alg:exp} and in Algorithm~\ref{alg:sincos}, and the vector $u$ of length $d+1$ in Algorithm~\ref{alg:monomials}
    \item similar to allocation on the stack, variables in the local scope of a function can be preserved over function calls by skipping connections, a technique which we have applied for the $x$ variable for all algorithms in \S\ref{s:examples}, making $x$ available both before and after the recursive function call (though local scoping is lost as $x$ is also available in between)
    \item the innermost for loop in Algorithm~\ref{alg:bitonic_sort} (with index $k$) for bitonic sorting could be realized in parallel, whereas the two outer for loops represent sequential operations which correspond to different layers
    \item the double for loop in Algorithm~\ref{alg:monomials} for $x^k$ could be realized as a matrix-vector product
\end{itemize}

The class of feedforward neural networks we have considered only allows quantities of fixed size, whose sizes are `known at compile time'. The network can implement allocation on the stack, but not on the heap. More flexibility is offered by recurrent neural networks, for which Turing completeness has been a subject of study~\cite{siegelmann1995computationalpower}.

It is evident that most algorithms can not be compiled to a neural network. However, if applicable, an algorithm clearly conveys what the network does and is amenable to analysis.

\subsection{What is the scope of neural network functions?}

We have focused on methodologies to implement an algorithm using a feedforward neural network with the ReLU activation function. It seems highly plausible that many more functions can be approximated in similar ways using folding and unfolding techniques. This includes functions of several variables.

Yet, all examples in~\S\ref{s:examples} have been special functions, for which shift relations are known analytically, which in turn lead to folding relations. It is an intriguing question whether recursive approximations can be found based on data, say based on function evaluations at the types of nested equispaced grids that we have considered.

There is another indication that the class of functions and algorithms we have considered are special cases. The networks in this paper are all sparse, or even highly sparse, in the sense that most of the weights are zero. Sparsity in the folding algorithms results from the definition of foldings between two consecutive layers only. In principle, the neural network has all the data from all finer layers at its disposal to return the next layer. In this context of function approximation from graded meshes, a dense network corresponds to linking the function to itself across multiple scales.

The approximation of monomials led to a more dense network than the other examples, yet it remains sparse overall. The algorithms and examples in this paper fall short of providing insight into the more general setting of dense deep neural networks. From the algorithmic perspective, dense networks would be both more multiscale in nature and more parallel than the examples in this paper.

\section*{Acknowledgements}

The author expresses his gratitude for enriching discussions on the topic of this paper with Simon Dirckx, Astrid Herremans and Yuji Nakatsukasa.

\bibliographystyle{abbrv}
\bibliography{references}

\end{document}